%% file: main.tex
\documentclass[letterpaper, 10 pt, conference]{ieeeconf}  % Comment this line out if you need a4paper

\IEEEoverridecommandlockouts                              % This command is only needed if 
\usepackage{bm}
\usepackage{nicefrac}
\usepackage{amssymb}
\usepackage{amsmath}

\usepackage{amsthm}
\usepackage{cite}
\usepackage{xcolor}
\usepackage{graphicx}
\usepackage{url}
\usepackage{svg}
\usepackage{balance}
\usepackage{multirow}
\usepackage{hyperref}
\hypersetup{
    colorlinks=true,
    linkcolor=black,
    filecolor=magenta,      
    urlcolor=cyan,
    pdfpagemode=FullScreen,
    }

\usepackage{times}
\usepackage{mathtools}

\usepackage{siunitx}
\usepackage{soul}
\newtheoremstyle{exampstyle}
  {3pt} % Space above
  {3pt} % Space below
  {\itshape} % Body font
  {} % Indent amount
  {\bfseries} % Theorem head font
  {.} % Punctuation after theorem head
  {.5em} % Space after theorem head
  {} % Theorem head spec (can be left empty, meaning `normal')

\theoremstyle{exampstyle} 
\newtheorem{definition}{Definition}
\newtheorem{lemma}{Lemma}
\newtheorem{theorem}{Theorem}
\newtheorem{remark}{Remark}
\newtheorem{assumption}{Assumption}
\newtheorem{problem}{Problem}

\theoremstyle{plain}

\definecolor{mumred}{RGB}{222,33,77}
\definecolor{mumgreen}{RGB}{0, 140, 0}
\definecolor{mumblue}{RGB}{0, 100, 222}
\definecolor{mumpurple}{RGB}{128, 0, 128}

\newcommand{\ofx}{\left(\bm{x}\right)}
\newcommand{\xb}{\bm{x}}
\newcommand{\ub}{\bm{u}}

\newcommand{\ddt}{\frac{\mathrm{d}}{\mathrm{d}t}}
\newcommand{\ddtau}{\frac{\mathrm{d}}{\mathrm{d}\tau}}

\DeclareMathOperator*{\argmin}{arg\,min}

\title{\LARGE \bf
Forward Invariance in Trajectory Spaces for Safety-critical Control
}

\author{Matti Vahs, Rafael I. Cabral Muchacho, Florian T. Pokorny and Jana Tumova% <-this % stops a space
	\thanks{This work was partially supported by the Wallenberg AI, Autonomous
		Systems and Software Program (WASP) funded by the Knut and Alice
		Wallenberg Foundation. This research has been carried out as part of the Vinnova Competence Center for Trustworthy Edge Computing Systems and Applications at KTH Royal Institute of Technology.
  }
	\thanks{The authors are with the Division of Robotics, Perception and Learning, KTH Royal Institute of Technology, Stockholm, Sweden and also affiliated with Digital Futures. Mail addresses: {\{\tt\small vahs, ricm, fpokorny, tumova\}}
		{\tt\small @kth.se}}%
}

\begin{document}
	\maketitle
	\thispagestyle{empty}
	\pagestyle{empty}

	%%%%%%%%%%%%%%%%%%%%%%%%%%%%%%%%%%%%%%%%%%%%%%%%%%%%%%%%%%%%%%%%%%%%%%%%%%%%%%%%
	\begin{abstract}
        % When considering the safety of robotic systems, it is crucial to ensure that control algorithms 
        Useful robot control algorithms should
        not only achieve performance objectives but also adhere to hard safety constraints.
        Control Barrier Functions (CBFs) have been developed to provably ensure system safety through forward invariance. However, they often unnecessarily sacrifice performance for safety since they are purely reactive.
        Receding horizon control (RHC), on the other hand, consider planned trajectories to account for the future evolution of a system. 
        This work provides a new perspective on safety-critical control by introducing Forward Invariance in Trajectory Spaces (FITS). 
        We lift the problem of safe RHC into the trajectory space and describe the evolution of planned trajectories as a controlled dynamical system. Safety constraints defined over states can be converted into sets in the trajectory space which we render forward invariant via a CBF framework. We derive an efficient quadratic program (QP) to synthesize trajectories that provably satisfy safety constraints. 
        Our experiments support that FITS improves the adherence to safety specifications without sacrificing performance over alternative CBF and NMPC methods.
        
	\end{abstract}

%%%%%%%%%%%%%%%%%%%%%%%%%%%%%%%%%%%%%%%%%%%%%%%%%%%%%%%%%%%%%%%%%%%%%%%%%%%%%%%%
\section{Introduction}
When deploying robots in the real world, we must ensure that they operate as intended and satisfy desired safety constraints such as obstacle avoidance. This gave rise to the field of safety-critical control where control laws are designed that keep the system provably safe. Prominent methods that have drawn a lot of attention in recent years are Control Barrier Functions (CBFs) \cite{ames2016control} which treat safety-critical control as an invariance problem. Specifically, CBFs define constraints on the system's state and are used to obtain a reactive controller that, if applied in continuous time, guarantees forward invariance of the constraint set. However, CBFs are purely reactive, which often results in sudden braking maneuvers as they do not take into account the future evolution of a system \cite{garg2024advances}. Therefore, performance is often unnecessarily sacrificed for safety.

Model Predictive Control (MPC) \cite{grune2017nonlinear}, on the other hand, reasons about the evolution of a dynamical system allowing for proactive behavior to avoid unsafe states. While this allows for a flexible formulation of safety-critical control, the resulting optimization is often nonlinear and nonconvex, which complicates efficient and rigorous control design.

Several existing works have addressed the issue of missing planning horizons in CBFs through path functions \cite{breeden2022predictive}, trajectory libraries \cite{infinitesimal2024}, or combinations of NMPC with discrete-time CBF constraints \cite{wabersich2022predictive, zeng2021safety}. Other works have looked at definitions of CBFs in alternative domains such as Belief Spaces \cite{vahs2023belief, vahs2024risk} which motivates the perspective taken in this work.
Specifically, our perspective on safety-critical receding horizon control (RHC) is based on the observation that planned trajectories typically do not vary a lot between two consecutive control steps. 
% Instead, two consecutive trajectories are rather close in trajectory space \rafael{remove this sentence?}.
% which is further supported by the fact that MPC problems are often warm started with the previous loop solution. 
Following this thought, we can think of the planned state trajectory $\mathcal{T}_x$ itself as a dynamical system that we can control, i.e. we can apply a change in the input trajectory of the system to affect the resulting state trajectory. This concept is highlighted in Fig.~\ref{fig:Intro} for a quadrotor system where a continuous change in the input trajectory $\dot{\mathcal{T}}_u$ leads to a continuous change in the state trajectory, i.e. $\dot{\mathcal{T}}_x$. 

Consequently, we can lift the problem of safe RHC into the trajectory space, i.e. the space of all possible planned trajectories. We formulate a control synthesis problem in trajectory space that guarantees the safe evolution of planned trajectories. Thus, we will show that the problem of safe RHC can be considered as an \textit{invariance problem in trajectory space.} Figure~\ref{fig:Intro} illustrates this concept where the current planned trajectory can be seen as a state in a higher dimensional trajectory space, indicated by the red point.

% we can think of the planned trajectory as a point with certain dynamics in the higher dimensional trajectory space and we construct a set that contains all safe trajectory states satisfying safety constraints over states. Thus, the problem of safe RHC can be converted into an invariance problem: \textit{if we keep the red point inside the safe set at all times, we ensure that all planned trajectories satisfy safety constraints.}
\begin{figure}[t]
    \centering
    \includegraphics[scale=1.1]{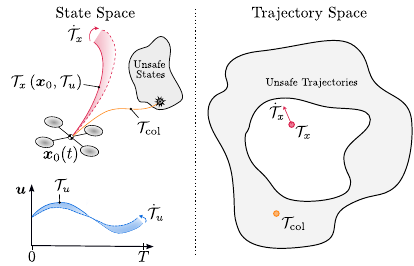}
    \caption{Illustration of trajectory spaces for a quadrotor where $\mathcal{T}_x$ represents the current planned trajectory in state space, $\mathcal{T}_u$ is the current control input trajectory of length $T$ and $\xb_0(t)$ is the current state of the quadrotor. The white set in trajectory space indicates the set of all safe planned trajectories and $\mathcal{T}_{\textrm{col}}$ is an exemplary unsafe trajectory.}
    \label{fig:Intro}
    \vspace{-0.5cm}
\end{figure}

We introduce \textit{forward invariance in trajectory spaces} (FITS), specifically, we
\begin{enumerate}
    \item overcome reactive behavior of CBFs by considering the future evolution of a system in a trajectory space,
    \item propose an efficient QP formulation in the trajectory space that guarantees safety over states proactively, and
    \item show in experiments that safety constraints can be enforced without major design effort.
\end{enumerate}
% In this paper, we argue that safe RHC can be seen as an invariance property in the space of possible trajectories. If we can find a notion of safe sets in trajectory spaces, we can, in combination with the dynamical system, guarantee 

\section{Related Work}
 We review two safety-critical control approaches relevant to our approach, namely CBFs and NMPC.

Predictive CBFs \cite{breeden2022predictive} have been introduced as a generalization of future-focused CBFs \cite{black2023future} to address the reactive nature of CBFs. The authors consider the evolution of a trajectory under a nominal control law which is allowed to be unsafe. Whenever this trajectory leads to future safety violations, the current control input is modified to proactively avoid unsafe states while assuming no actuation constraints. In contrast, we formulate safe sets in trajectory space that encode both, safety and actuation constraints.

Backup CBFs \cite{gurriet2018online, chen2021backup} guarantee the existence of a trajectory that will lead to a control invariant set. Specifically, to ensure future safety, \cite{gurriet2018online} leverages a flow map, i.e. the solution to an initial value problem of an ODE under the safe backup policy. The main difficulty is to find a valid backup policy which typically requires domain specific knowledge. In \cite{wiltz2023construction}, Hamilton-Jacobi reachability analysis is used to find a valid backup policy. Our approach does not require the definition of a backup policy. Instead, we directly optimize over trajectory changes to ensure future safety.

% While backup CBFs only use the gradient of the flow map wrt. initial conditions and assume a constant backup policy, we also leverage the gradient wrt. control inputs to change the planned trajectory. The main difficulty is to find a valid backup policy which typically requires domain specific knowledge. In \cite{wiltz2023construction}, Hamilton-Jacobi reachability analysis is used to find a valid backup policy. 

To combine CBF and RHC approaches, \cite{zeng2021safety} introduced discrete-time CBF conditions as constraints in an NMPC problem. If a valid discrete-time CBF is available, the approach can guarantee safety along a horizon given that the optimization problem is always feasible. 
% However, the NMPC problem is generally nonlinear and nonconvex which complicates the design of recursively feasible controllers. 
An empirical investigation has been carried out in \cite{zeng2021enhancing} which shows that the CBF condition generally enhances the feasibility. 
The authors of \cite{wabersich2022predictive} leverage model predictive safety filters that enforce terminal states of planned trajectories to lie within a control invariant set which is closely related to backup CBFs. 
However, NMPC tends to be computationally expensive as it requires solving a nonlinear program, leading to the introduction of more efficient yet inexact methods \cite{gros2020linear, zanelli2021inexact, numerow2024inherently}.

To overcome the computational burden of NMPC, continuation methods \cite{ohtsuka2004continuation, shen2016modified} have been introduced for RHC. 
These methods address problems with equality constraints exclusively by formulating a system of equations to determine the optimal control trajectory. The authors of \cite{ohtsuka2004continuation} exploit the recursive structure of RHC and look at the time derivative of the system of equations. This formulation leads to a dynamical system that can be efficiently integrated to obtain a control law. As the continuation method only allows for equality constraints, safety and actuation constraints which are formulated as inequalities cannot be handled. In contrast, we model the planned trajectories as dynamical system which allows for inequality constraints on the trajectory state.

% Most relevant to our approach are integral CBFs (I-CBFs) \cite{ames2020integral} and receding horizon control in continuous time \cite{infinitesimal2024}. The former work introduces a coupled dynamical system of states and controls to find a safe integral control input that satisfy both, safety and actuation 
% constraints, while acting in a reactive manner. 
In \cite{infinitesimal2024}, parameterized trajectory libraries are used to reformulate a RHC problem into a parameter optimization. Finding a trajectory library that satisfies the assumptions of always ending in an equilibrium state \cite{infinitesimal2024} is difficult for general nonlinear systems. In our work, we directly optimize for changes in the input trajectory which makes it applicable to a broader class of systems.

\input{content/preliminaries_C}

\section{Problem Statement}
A trajectory $\mathcal{T}_x^I(t)$ that is planned in a receding horizon fashion should be safe at all times $t$ which we formally define hereafter.
\begin{problem}
Given a dynamical system in Eq.~\eqref{eq:control_system_dynamics} and a planning horizon $I=[0, T]$, ensure that $\forall t \geq 0$
\begin{align}
    h_i(\bm{x}_{\tau}) &\geq 0,~\forall i=0,\dots, m-1, \forall \bm{x}_{\tau} \in \mathcal{T}^I_x(t),\label{eq:rhc_constraints}\\
    \ub_{\mathrm{min}} &\leq \ub_{\tau} \leq \ub_{\mathrm{max}} \hspace{1.62cm}\forall \ub_{\tau} \in \mathcal{T}^I_u(t), \label{eq:input_constraints}
\end{align}
where $m$ inequality constraints encode safety specifications and Eq.~\eqref{eq:input_constraints} denotes actuation constraints.
\end{problem}
% One possible solution to Problem 1 is to ensure that the OCP in Eq.~\eqref{eq:rhc_cost}-\eqref{eq:rhc_initial} is always feasible which is difficult since it is a generally nonlinear and nonconvex optimization problem. 
\noindent We want to solve this problem 
% through the lens
through the framework
of control barrier functions to benefit from their formal guarantees. To that end, we lift the problem to the trajectory space, i.e. the space of all possible planned trajectories.
\subsection{Trajectory Space Dynamics}
We introduce the continuous-time dynamics of trajectories $\mathcal{T}_x^I$ and $\mathcal{T}_u^I$, respectively.
%n our work, we decide to use an integral control law of the form 
%\begin{align*}
%    \ddt{\bm{u}}_{\tau}(t) = \dot{\bm{u}}_{\tau}(t) = \bm{v}_{\tau}(t), %\hspace{0.5cm} \forall \tau \in I
%\end{align*}
%meaning that we control the change of physical control inputs $\bm{u}$ through a virtual control input $\bm{v}$. Here, we introduce the ${(\cdot)}$ notation as the derivative with respect to actual time $t$, not to be confused with planning time $\tau$. 
% The state trajectory dynamics and the used notation are visualized in Figure 

% Figure~\ref{fig:state_trajectory_space} visualizes the notation and introduced dynamics through an example on a  one dimensional state space.

To study the change of the trajectories over time we introduce the derivative along system time $t$ of the corresponding state trajectory $\mathcal{T}_x^I(t)$, reading 
% \begin{align}
%     \dot{\mathcal{T}}_x^I(t) &= \dot{\Phi}\left(\xb_0, \mathcal{T}_u^{I}\right) = \frac{\partial \Phi}{\partial \xb_0} \dot{\xb}_0(t) + \frac{\partial \Phi}{\partial \mathcal{T}_u^{I}} \dot{\mathcal{T}}_u^{I}(t),\label{eq:Tx}\\
%     \dot{\mathcal{T}}_u^I(t) &= \mathcal{T}_v^I(t) = \left\{\bm{v}_{\tau}(t) \in \mathbb{R}^{n_u} \mid \tau \in I\right\} \in \mathcal{V}^I,\label{eq:Tu}
% \end{align}
% Alternative notation:
\begin{align}
    \dot{\mathcal{T}}_x^I(t) &\coloneq \left\{\ddt \varphi\left(\xb_0(t), \mathcal{T}_u^{[0, \tau]}(t)\right) \mathrel{\Big|} \tau \in I\right\}\label{eq:Tx}\\
    &= \left\{\frac{\partial \varphi}{\partial \xb_0(t)} \dot{\xb}_0(t) + \int_{0}^{\tau}\frac{\partial \varphi}{\partial \ub_{\hat{\tau}}(t)} \dot{\ub}_{\hat{\tau}}(t) \, \mathrm{d}\hat{\tau} \mathrel{\Big|} \tau \in I\right\} \label{eq:partials} \\
    \dot{\mathcal{T}}_u^I(t) &= \mathcal{T}_v^I(t) = \left\{\bm{v}_{\tau}(t) \in \mathbb{R}^{n_u} \mid \tau \in I\right\} \in \mathcal{V}^I,\label{eq:Tu}
\end{align}
where we use the ${(\cdot)}$ notation for the derivative with respect to actual time $t$, and $\mathcal{T}_v^I$ denotes the trajectory of a virtual input $\bm{v}$. This virtual input controls the rate of change of the input trajectory, i.e. $\dot{\ub}_{\tau}(t)=\bm{v}_{\tau}(t)$.
The existence of the partial derivatives in the definition \eqref{eq:partials} follows from the assumption on continuously differentiable dynamics \eqref{eq:control_system_dynamics}. 
% which describes the dynamical system of $\mathcal{T}_x$ and $\mathcal{T}_u$ as a function of the virtual control input trajectory $\mathcal{T}_v$ that we seek to design to achieve desirable properties of the aforementioned trajectory system. 
% which enables the definition of trajectory space dynamics, using $v_I$ as an input
% \begin{align}
%     \ddt \begin{bmatrix}
%         x_I \\
%         u_I
%     \end{bmatrix} = 
%     \begin{bmatrix}
%         F_d(x_0, u_I) \\
%         0
%     \end{bmatrix} + \begin{bmatrix}
%         F_c(x_0, u_I) \\
%         1
%     \end{bmatrix} v_I.\label{eq:TrajectoryDynamics}
% \end{align}

% (A problem statement that build on this version of preliminaries, can formalize the notation we continue to use in the core sections, and pose the problem as: ensure safety constraints defined as inequality constraints on state and inputs along planned trajectories in a receding horizon interval for all times. Then the method and main contributions are both the presentation of this perspective and the computationally-viable definition of the state, and of state and input constraints as sets in trajectory space,) 
Thus, we view the planned trajectories $\mathcal{T}_x^I$ and $\mathcal{T}_u^I$ as dynamical systems that we can affect through the virtual input $\bm{v}$.
% in order to achieve desirable properties of the trajectory system. 
Having introduced this dynamical system, we can reformulate the original problem as follows.
\begin{problem}
    Given the dynamical system of trajectories in Eq.~\eqref{eq:Tx}-\eqref{eq:Tu} and
    % a set of safety constraints as defined in Eq.~\eqref{eq:rhc_constraints} with actuation constraints~\eqref{eq:input_constraints},
    safety constraints \eqref{eq:rhc_constraints}-\eqref{eq:input_constraints},
    find a safe set $\mathcal{C}$ in trajectory space $\mathcal{X} \times \mathcal{U}^I$ and a corresponding control law $\pi_v: \mathcal{X} \times \mathcal{U}^I \mapsto \mathcal{V}^I$ that ensures the satisfaction of safety constraints by rendering the safe set forward invariant.
    % find a control invariant set $\mathcal{C}_h$ in trajectory space with the corresponding control law $\pi_v: \mathcal{X} \times \mathcal{U}^I \mapsto \mathcal{V}^I$ such that forward invariance of $\mathcal{C}_h$ satisfies all safety and actuation constraints. 
\end{problem}
By solving Problem 2, we ensure that the planned trajectories remain in a safe set in trajectory space at all times which, in turn, would solve Problem 1.
% Note that, we do not require the trajectories to be optimal with respect to the objective function defined in Eq.~\eqref{eq:rhc_cost}. 
% However, we will discuss in section \ref{sec:objective}, how an objective can be incorporated into the control law $\pi_v$.
% However, an objective can be incorporated into the control law $\pi_v$ for guidance as described in section~\ref{sec:objective}.

% In this work, we want to solve a general OCP for arbitrary smooth dynamics and constraints as in Eq. REF and we want to get strong safety guarantees as in Bla. Further, we consider continuous-time dynamical system while treating the control input in discrete-time with a ZOH.
% \begin{assumption}
%     If the OCP is solved in continuous time, then the input trajectory changes continuously as well, i.e. the input trajectory does not have jumps over time.
% \end{assumption}

\begin{figure}
    \centering
    \includegraphics[width=\linewidth]{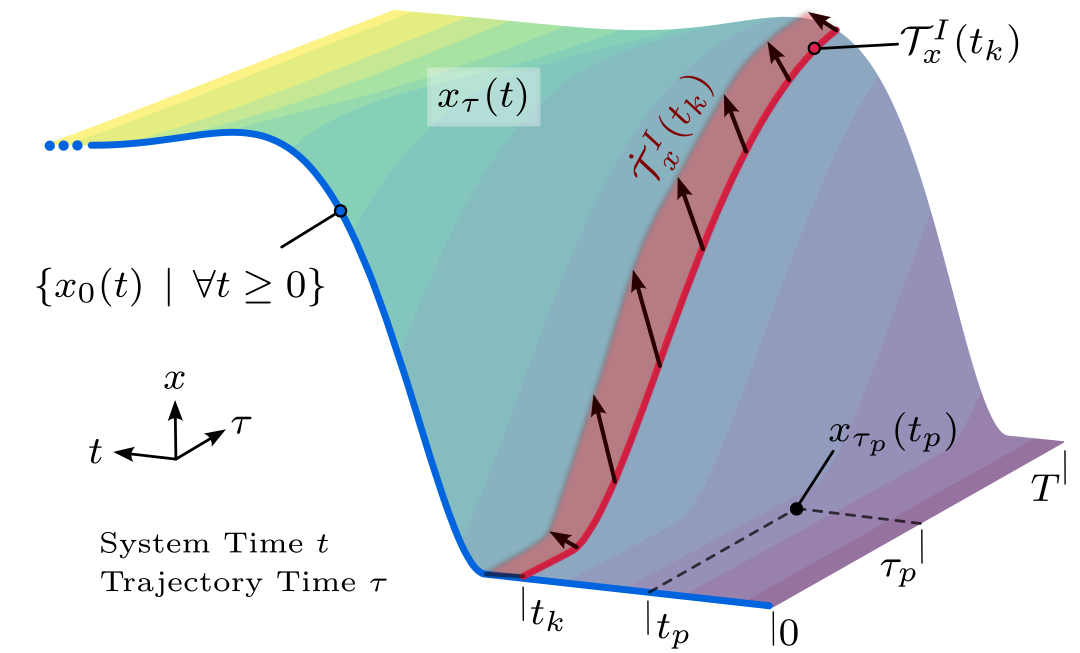}
    \caption{Notation used for time variables and the trajectory space dynamics for a one-dimensional state space example. The surface color changes along the system time $t$. We denote by $x_{\tau_p}(t_p)$ the state at system time $t_p$ and trajectory time $\tau_p$. The time derivative of the state trajectory at time $t_k$ within the planning interval, i.e., $\dot{\mathcal{T}}_x^I(t_k)$, is shown as a vector field (shaded in red) on the corresponding state trajectory $\mathcal{T}_x^I(t_k)$. The executed trajectory $\{x_0(t) \mid \forall t \geq 0 \}$ is shown in blue.}
    \vspace{-0.5cm}
    \label{fig:state_trajectory_space}
\end{figure}

\section{Forward Invariance in Trajectory Spaces}
To solve the aforementioned problem, we introduce Forward Invariance in Trajectory Spaces (FITS) which treats receding horizon control as a dynamical system in the trajectory space. Our approach can be summarized as follows: We 1) define a trajectory configuration space $\mathcal{S}$ with a finite dimensional state $\bm{s} \in \mathcal{S}$ that uniquely describes the current planned trajectory $\mathcal{T}_x^I$ and obtain its dynamics, 2) we convert the original safety and actuation constraints into sets in the space $\mathcal{S}$ and 3) we ensure that these sets are forward invariant.

\subsection{Trajectory State Parameterization}
We introduced the dynamics of planned trajectories as a dynamical system over continuous functions in Eq.~\eqref{eq:Tx}-\eqref{eq:Tu} which describes an infinite dimensional dynamical system. Therefore, the goal of this section is to find a finite representation that we can use for control design.

% To construct an input trajectory $\mathcal{T}_u^I$, we use a finite parameterization. 
% Specifically, 
We consider 
% a zero-order-hold (ZOH) 
% discretization 
% parameterization
% of 
piecewise constant input trajectories of $N$ steps on the interval $I=[0, T]$
\begin{align*}
    \mathcal{T}_u^I(t) = \left\{\ub_i\mid \forall \tau \in [\tau_i, \tau_{i+1}), \forall i = 0, \dots, N-1\right\} \in \mathcal{U}^I
\end{align*}
where $\bm{u}_i \in \mathcal{U}$ is the input applied within the $i$th step, each for a duration of $\delta \tau = \nicefrac{T}{N}$. 
% Consequently, the infinite dimensional input trajectory is reduced to a finite set of control inputs. 
Following this parameterization of an input trajectory, we define the state
\begin{align}
    \bm{s}(t) := \begin{bmatrix}
        \xb_0(t) & \ub_0(t) & \dots & \ub_{N-1}(t)
    \end{bmatrix}^T \in \mathcal{S} \subseteq \mathcal{X} \times \mathcal{U}^N,
\end{align}
that uniquely describes the current planned trajectory $\mathcal{T}_x^I(t)$ since it only depends on the initial conditions $\bm{x}_0(t)$ and the input trajectory $\mathcal{T}_u^I(t)$. 

Next, we define a dynamical system that describes the evolution of $\bm{s}$. The state of the robotic system $\xb_0(t)$ is always the initial state for a planned trajectory $\mathcal{T}_x^I(t)$. Thus, its total derivative wrt. system time is given by the dynamics model~\eqref{eq:control_system_dynamics}, i.e. $\dot{\xb}_0(t)=\bm{F}({\xb}_0(t), {\ub}_0(t))$.
Since the control input trajectory is a decision variable that we can change over time, we model it as a single integrator, i.e. $\dot{\bm{u}}_i = \bm{v}_i$ where $\bm{v}_i$ is a virtual input. 
% that controls the change of $\bm{u}_i$. 

In summary, a trajectory state $\bm{s} \in \mathcal{S}$ evolves according to 
% The vector $\bm{s}$ describes a point in the trajectory configuration space $\mathcal{S}$, and its dynamics are described by 
\begin{align}
    \dot{\bm{s}} &= \begin{bmatrix}
        \bm{F}(\bm{x}_0, \ub_0)\\
        \bm{0}
    \end{bmatrix} + \begin{bmatrix}
        \bm{0}\\
        \bm{I}
    \end{bmatrix}\bm{v}\\
    &= \bm{f}_s(\bm{s}) + \bm{g}_s(\bm{s}) \bm{v}
\end{align}
where $\bm{v} \in \mathbb{R}^{n_u \cdot N}$.
Interestingly, 
even though the dynamics of $\bm{x}$ do not need to be control affine, the dynamics of the trajectory state are 
% always 
affine in $\bm{v}$.

\subsection{State Constraints as Safe Set in Trajectory Space}
In this section, we describe how the safety constraints in Eq.~\eqref{eq:rhc_constraints} can be converted into a set $\mathcal{C}_h \subseteq \mathcal{S}$ that, if forward invariant, ensures that the original constraints hold. 
The main difficulty lies in the fact that safety constraints are defined over physical states $\xb$ while $\mathcal{S}$ only contains the initial physical state $\xb_0$ and a piecewise constant input trajectory $\mathcal{T}_u^I$.

Given a current trajectory state $\bm{s}$, we can leverage the mapping $\varphi$ defined in Eq.~\eqref{eq:ODEInt}, to map a state $\bm{s}$ to a physical state $\xb_{\tau}$ on the trajectory $\mathcal{T}_x^I$, reading
%With slight abuse of notation, 
% We change the arguments of this mapping to
%there exists a mapping $\bm{\varphi}$ from input trajectory space $\mathcal{S}$ to state trajectory space $\mathcal{T}_x$ meaning the space of continuous state trajectories of length $T$. This mapping is given by the ODE solution
\begin{align}
    \xb_{\tau} = \varphi\left(\xb_0, \mathcal{T}_u^{[0, \tau]}\right) = \varphi(\bm{s}, \tau). \label{eq:integration}
\end{align}
Thus from a state $\bm{s}$ we can uniquely obtain a continuous time trajectory $\mathcal{T}_x^I$ by solving the ODE in Eq.~\eqref{eq:control_system_dynamics} given the initial state and the continuous-time input trajectory. More importantly, we can leverage efficient differentiable ODE solvers \cite{kidger2021hey} to also obtain the Jacobian $\nicefrac{\partial {\varphi}}{\partial \bm{s}}$.

% Given a state trajectory $\xb(t)$, we can define arbitrary state constraints such as
% \begin{align}
    % \bm{h}(\xb(t)) \geq 0, \hspace{0.5cm} \forall t \in [t_0, t_0 + T]
% \end{align}
% which could, for example, encode obstacle avoidance for mobile robots. Note that this constraint function over physical states essentially describes a safe set in input trajectory space $\mathcal{S}$ as 
Next, we can define a set in the trajectory configuration space $\mathcal{S}$ that encodes the constraints in Eq.~\eqref{eq:rhc_constraints} as
\begin{align*}
    \mathcal{C}_h = \left\{\bm{s} \in \mathcal{S} \mid h_i(\varphi(\bm{s}, \tau)) \geq 0 , \hspace{0.2cm} \forall \tau \in I, \forall i=0, \dots, m-1\right\}\label{eq:SafeSetExact}
\end{align*}
since the physical state $\bm{x}_{\tau}$ is explicitly expressed as a function of the state $\bm{s}$ through the mapping $\varphi$.
\begin{assumption}
    % Since we consider nonlinear systems, an ODE solver $\varphi$ can only return a discretized solution of $\xb_{\tau}$ 
    We consider a discretized solution of $\xb_{\tau}$ at times $\tau_0, \dots, \tau_{M-1}$ where $\tau_0 =0$ and $\tau_{M-1} = T$, respectively. We assume that it is sufficient to consider constraints at discrete time steps along the planned trajectory $\mathcal{T}_x^I$, i.e.
    \begin{align*}
        \mathcal{C}_h = \bigcap_{j=0}^{M-1} \left\{\bm{s} \in \mathcal{S} \mid h_i(\varphi(\bm{s}, \tau_j)) \geq 0, \hspace{0.2cm} \forall i = 0,\dots,m-1\right\}.
    \end{align*}
\end{assumption}
While this approximation does not recover the original safe set over continuous trajectories, we can get arbitrarily close by increasing the number of time steps considered. Further, we leave it for future work to find Lipschitz arguments to relax Assumption 1.

To ensure that a planned trajectory will satisfy safety constraints at all times, we formulate a CBF condition in trajectory configuration space $\mathcal{S}$.
\begin{lemma}
If for all $\bm{s} \in \mathcal{C}_h$, there exists a locally Lipschitz virtual input $\bm{v} \in \mathbb{R}^{n_u \cdot N}$ 
% $\bm{s} \in \mathcal{S}$  
and a class-$\kappa$ function $\alpha_1$ such that $\forall i = 0,\dots, m-1$ (constraints) and $\forall j = 0,\dots, M-1$ (time steps)
\begin{equation}
\begin{aligned}
    \ddt\left(h_i(\varphi_j)\right) &= \frac{\partial h_i}{\partial \xb} \frac{\partial \varphi_j}{\partial \bm{s}} \left(\ddt \bm{s}\right)\\
    &= \frac{\partial h_i}{\partial \xb} \frac{\partial \varphi_j}{\partial \bm{s}} \left(\bm{f}_s(\bm{s}) + \bm{g}_s(\bm{s}) \bm{v}\right)\\
    &\geq - \alpha_1\left(h_i(\varphi_j)\right)
\end{aligned}
\label{eq:SafeDT}
\end{equation}
with $\varphi_j = \varphi(\bm{s}, \tau_j)$ holds, then $\mathcal{C}_h$ is forward invariant.
%meaning that $\bm{s}(t_0) \in \mathcal{S} \implies \bm{s}(t) \in \mathcal{S}, \hspace{0.1cm} \forall t \geq t_0$.
\end{lemma}
The proof follows analogously to the proof of forward invariance of zeroing CBFs in Theorem~\ref{thm:CBF}, proposed in \cite{ames2016control}. Importantly, the invariance of $\mathcal{C}_h$ implies that the inequality constraints over physical states in Eq.~\eqref{eq:rhc_constraints} holds at all times.
\begin{remark}
    While the continuous-time physical state trajectory $\mathcal{T}_x^I(\bm{s})$ can have points of non-smoothness, every point on the trajectory $\xb_{\tau_i}=\varphi(\bm{s}, \tau_i)$ is continuously differentiable with respect to the state $\bm{s}$. This is necessary for the invariance result and follows from the continuously differentiable dynamics model~\eqref{eq:control_system_dynamics}.
\end{remark}

\subsection{Actuation Constraints}
In order to ensure safety in real-world settings, we also need to consider actuation constraints of robots. In FITS, actuation constraints can be interpreted as sets in $\mathcal{S}$ that we want to render forward invariant since the control input $\bm{u}$ is part of the state $\bm{s}$.

Similar to the considered safety constraints, we can directly define sets in $\mathcal{S}$ that ensure actuation constraints to be satisfied, i.e.
\begin{align*}
    \mathcal{C}_{\mathrm{min}} &= \left\{\bm{s} \in \mathcal{S} \mid \bm{h}_{\mathrm{min}}(\bm{s}) \geq \bm{0}\right\}\\
    \mathcal{C}_{\mathrm{max}} &= \left\{\bm{s} \in \mathcal{S} \mid \bm{h}_{\mathrm{max}}(\bm{s}) \geq \bm{0}\right\},
\end{align*}
% \begin{align*}
%     \mathcal{C}_{\mathrm{min}} &= \left\{\bm{s} \in \mathcal{S} \mid \ub_i - \ub_{\mathrm{min}} \geq 0, ~\forall i = 0,\dots, N-1\right\}\\
%     \mathcal{C}_{\mathrm{max}} &= \left\{\bm{s} \in \mathcal{S} \mid \ub_{\mathrm{max}} - \ub_i \geq 0, ~\forall i = 0, \dots, N-1\right\}.
% \end{align*}
where $\bm{h}_{\mathrm{min}} = [
        (\ub_0 - \ub_{\mathrm{min}})~\dots~(\ub_{N-1} - \ub_{\mathrm{min}}) 
    ]^T$ and $\bm{h}_{\mathrm{max}} = [
        (\ub_{\mathrm{max}} - \ub_0)~\dots~(\ub_{\mathrm{max}} - \ub_{N-1}) 
    ]^T$.
Thus, to satisfy actuation constraints, we need to ensure that there exists a virtual input $\bm{v}$ such that $\forall k =0, \dots, n_u \cdot N-1$
\begin{equation}
\begin{aligned}
    \ddt\left(h_{\mathrm{min}}^{(k)}(\bm{s})\right) &= \frac{\partial h_{\mathrm{min}}^{(k)}}{\partial \bm{s}}\left(\bm{f}_s(\bm{s}) + \bm{g}_s(\bm{s}) \bm{v}\right)\\
    &\geq - \alpha_2\left(h_{\mathrm{min}}^{(k)}(\bm{s})\right)
\end{aligned}
\label{eq:ActuationDT}
\end{equation}
holds, where $\alpha_2$ is a class-$\kappa$ function and $h_{\mathrm{min}}^{(k)}$ is the $k$th entry of $\bm{h}_{\mathrm{min}}$. The condition for the maximum actuation constraints follows analogously.
\subsection{Incorporating an Objective} \label{sec:objective}
% Since our proposed method offers a planning horizon, we can also reason about costs on the physical state trajectory. 
Contrary to standard CBF-QP formulations as in Eq.~\eqref{eq:QP}, FITS does not need an explicit reference controller and can directly incorporate a cost function that, e.g., encodes desired tracking performance.

For a given state $\bm{s}$, we can reconstruct both, the input trajectory $\mathcal{T}_u^I$ and the state trajectory $\mathcal{T}_x^I$. Thus, we can define an objective function over input and state trajectories as $J = J(\mathcal{T}_x^I, \mathcal{T}_u^I) = J(\bm{s})$. More importantly, if $J$ is differentiable, we can evaluate the time derivative of the objective function as 
\begin{equation}
\begin{aligned}
    \frac{\mathrm{d}}{\mathrm{d}t} J(\bm{s}) = \frac{\partial J}{\partial \bm{s}} \dot{\bm{s}}
    = \frac{\partial J}{\partial \bm{s}} \left(\bm{f}_s(\bm{s}) + \bm{g}_s(\bm{s}) \bm{v}\right),
\end{aligned}
\label{eq:CostDT}
\end{equation}
meaning that we can affect the control trajectory via $\bm{v}$ such that the objective will be minimized over time. Generally, the objective function can encode any desired behavior as function of states and inputs but we can also recover the same objective as in Eq.~\eqref{eq:QP} where the deviation of a reference controller is minimized.
\subsection{Control Synthesis in Trajectory Spaces}
This leads us to our final formulation of FITS in which we optimize for a virtual input $\bm{v}$ that minimizes an objective function over time while satisfying actuation and safety constraints.
% can be interpreted as a CBF in trajectory spaces. Instead of directly solving for a control input trajectory, we see it as a dynamical system that we are controlling over time. 

Given an initial state $\bm{s}$ composed of the initial physical state $\xb_0(t_0)$ as well as an initial input trajectory $\mathcal{T}_u^I(t_0)$, we synthesize virtual inputs $\bm{v}$ by solving
% \begin{equation}
% \begin{aligned}
%     \bm{v}^* = \argmin_{\bm{v}}\quad & \text{Eq.}\, \eqref{eq:CostDT} + \bm{v}^T \bm{Q}\bm{v}\\
%     \textrm{s.t.~~} \quad & \text{Eqns.}\, \eqref{eq:SafeDT}, \eqref{eq:ActuationDT},
% \end{aligned}
% \label{eq:QPTrajectorySpace}
% \end{equation}
\begin{equation}
\begin{aligned}
    \min_{\bm{v}}\quad & \bm{v}^T \bm{Q}\bm{v} + \frac{\partial J}{\partial \bm{s}}\bm{g}_s(\bm{s}) \bm{v}\\
    \textrm{s.t.~~} \quad & \frac{\partial h_i(\varphi_j)}{\partial \bm{s}} \left(\bm{f}_s(\bm{s}) + \bm{g}_s(\bm{s}) \bm{v}\right)
    \geq - \alpha_1\left(h_i(\varphi_j)\right)\\
     \quad & \frac{\partial h_{\mathrm{min}}^{(k)}}{\partial \bm{s}}\left(\bm{f}_s(\bm{s}) + \bm{g}_s(\bm{s}) \bm{v}\right)
    \geq - \alpha_2\left(h_{\mathrm{min}}^{(k)}(\bm{s})\right)\\
    \quad & \frac{\partial h_{\mathrm{max}}^{(k)}}{\partial \bm{s}}\left(\bm{f}_s(\bm{s}) + \bm{g}_s(\bm{s}) \bm{v}\right)
    \geq - \alpha_2\left(h_{\mathrm{max}}^{(k)}(\bm{s})\right)
\end{aligned}
\label{eq:QPTrajectorySpace}
\end{equation}
with $\varphi_j \coloneq \varphi(\bm{s}, \tau_j)$, $i=0,\dots, m-1$, $j=0,\dots, M-1$ and $k=0,\dots, n_u \cdot N-1$. Note that the optimization problem \eqref{eq:QPTrajectorySpace}
is quadratic in $\bm{v}$ and, thus, can be solved efficiently using off-the-shelf QP solvers. 
The additive quadratic term in $\bm{v}$ with a positive definite matrix $\bm{Q}$ reduces the rate of change of the input trajectory over time, inducing a smoother change in trajectories. 
To connect our method back to the problem statement, the QP in Eq.~\eqref{eq:QPTrajectorySpace} is a control law $\pi_v: \mathcal{S} \to \mathcal{V}^I$ that we propose as solution to Problem 2.
% , \bm{s}\mapsto \mathrm{QP}^{\eqref{eq:QPTrajectorySpace}}(s)$ that solves Problem 2.

% \rafael{An alternative to the above CBF-QP approach is through Linear Programming (LP), by indirectly constraining $\bm{v}$ through an inequality constraint, limiting the rate of change of the trajectory tail end $\xb(T)$
% \begin{align}
%     \dot{\xb}_\mathrm{min} \leq \dot{\xb}_0 + \frac{\partial \varphi(\bm{s}, T)}{\partial \ub} \bm{v} \leq \dot{\xb}_\mathrm{max}, \label{eq:tail_ineq}
% \end{align}
% leading to the LP
% \begin{equation}
% \begin{aligned}
%     \mathrm{min}_{\bm{v}} \quad & \left( \frac{\partial J}{\partial \bm{s}}\bm{g}_s(\bm{s}) \right)^T\bm{v}\\
%     \textrm{s.t.~~} \quad & \text{Eqns.} \eqref{eq:SafeDT}, \eqref{eq:ActuationDT}, \eqref{eq:tail_ineq}.
% \end{aligned}
% \label{eq:LPTrajectorySpace}
% \end{equation}

% The behaviors arising from the QP and LP formulations are expected to be similar, although not equivalent. According to literature (chatgpt) the LP is faster and scales better than the QP. Further, note that the QP becomes an LP when removing the quadratic input cost penalty. An empirical comparison of the formulations wrt. computation and convergence could be interesting, e.g., using \texttt{cvxopt} that provides QP and LP implementations with similar calls.

% }

\section{Analysis}
% \subsection{Soundness}
We now have the necessary elements to discuss the theoretical properties of FITS.
\begin{theorem}
    Consider the dynamical system in Eq.~\eqref{eq:control_system_dynamics}, the safety constraints in Eq.~\eqref{eq:rhc_constraints} and the actuation constraints in Eq.~\eqref{eq:input_constraints}. If the control input sequence evolves according to
    \begin{align*}
        \dot{\bm{u}}_i = \bm{v}_i^*, \hspace{0.2cm} \forall i = 0,\dots, N-1
    \end{align*}
    where $\bm{v}^*$ is the solution to Eq.~\eqref{eq:QPTrajectorySpace}, then all inequality constraints \eqref{eq:SafeDT} and actuation constraints \eqref{eq:ActuationDT} will be satisfied at all times $t \geq t_0$, given that $\bm{s}(t_0) \in \mathcal{C}_h \cap \mathcal{C}_{\textrm{min}} \cap \mathcal{C}_{\textrm{max}}$.
\end{theorem}
\begin{proof}
    If the QP is always feasible, Eq.~\eqref{eq:SafeDT} always holds, implying that the set $\mathcal{C}_h$ is forward invariant. Similarly, the sets defining the actuation constraints, $\mathcal{C}_{\mathrm{min}}$ and $\mathcal{C}_{\mathrm{max}}$, are rendered forward invariant through Eq.~\eqref{eq:ActuationDT} and, thus, actuation constraints are satisfied at all times, see Thm.~\ref{thm:CBF}.
\end{proof}
\subsection{Feasibility of the QP}
% In general, especially for CBFs, it is difficult to prove recursive feasibility of a CBF-QP because it requires us to show that a CBF candidate function is valid under input saturation, i.e. there exists a $\ub$ in the bounded control set $\mathcal{U}$ such that Def.~\eqref{def:CBF} holds for all states in the safe set.

In FITS, proving recursive feasibility is different than in general input-constrained CBF-QPs. We have to show that multiple safe sets are control-sharing since all sets are defined in the same space and can be affected via $\bm{v}$. Here, control-sharing means that the individual CBF conditions of safety and actuation are non-conflicting.  However, so far there is only work for single input systems that rigorously provides a controller with feasibility guarantees for multiple defined safe sets \cite{xu2018constrained}. 

Below, we show the necessary conditions for the QP to be feasible.
\begin{theorem}
\label{thm:feasibility}
    The QP in Eq.~\eqref{eq:QPTrajectorySpace} is always feasible if and only if
    \begin{align}
        \sum_k \lambda_k \frac{\partial h_k}{\partial \bm{s}} \bm{g}_s = \bm{0} \implies  \sum_k \lambda_k \left(\frac{\partial h_k}{\partial \bm{s}} \bm{f}_s + \alpha(h_k(\bm{s}))\right) \geq 0\label{eq:thmQP}
    \end{align}
    holds for each individual constraint in Eq.~\eqref{eq:QPTrajectorySpace}, i.e. $k=1,\dots,m\cdot M+ 2\cdot n_u \cdot N$ and for all $\bm{s} \in \mathcal{C}_h \cap \mathcal{C}_{\textrm{min}} \cap \mathcal{C}_{\textrm{max}}$ and for all $\lambda_k \geq 0$.
\end{theorem}
\begin{proof}
    Consider the feasibility problem of the QP in Eq.~\eqref{eq:QPTrajectorySpace} which is characterized by the Lagrangian
    \begin{align*}
        \mathcal{L}(\bm{s}, \bm{v}, \bm{\lambda}) = -\sum_k \lambda_k \left(\frac{\partial h_k}{\partial \bm{s}} \bm{g}_s \bm{v} + \frac{\partial h_k}{\partial \bm{s}} \bm{f}_s + \alpha(h_k(\bm{s}))\right).
    \end{align*}
    If there exists a $\bm{v} \in \mathbb{R}^{n_u \cdot N}$ for all non-negative Lagrange multipliers $\bm{\lambda}\geq \bm{0}$ such that the complementary slackness condition $\mathcal{L}(\bm{s}, \bm{v}, \bm{\lambda})\leq 0$ is satisfied, the QP is feasible for a given state $\bm{s} \in \mathcal{S}$ \cite{boyd2004convex}. 
    
    If Eq.~\eqref{eq:thmQP} holds, it follows directly that $\mathcal{L} \leq 0$ if $\sum_k \lambda_k (\nicefrac{\partial h_k}{\partial \bm{s}}) \bm{g}_s = \bm{0}$ holds. Otherwise, it follows that the Lagrange dual function results in $\mathrm{inf}_{\bm{v}}~ \mathcal{L}(\bm{s}, \bm{v}, \bm{\lambda}) = - \infty$ which concludes that Eq.~\eqref{eq:thmQP} implies feasibility of the QP for a given state $\bm{s}$. Since Eq.~\eqref{eq:thmQP} holds for all $\bm{s} \in \mathcal{C}_h \cap \mathcal{C}_{\textrm{min}} \cap \mathcal{C}_{\textrm{max}}$, the QP is feasible for all $t \geq t_0$ given that $\bm{s}(t_0) \in \mathcal{C}_h \cap \mathcal{C}_{\textrm{min}} \cap \mathcal{C}_{\textrm{max}}$ since $\mathcal{C}_h \cap \mathcal{C}_{\textrm{min}} \cap \mathcal{C}_{\textrm{max}}$ is forward invariant (see Thm.~2).
\end{proof}
In words, Thm.~\ref{thm:feasibility} requires the uncontrolled solution of $\bm{s}$ according to the drift term $\bm{f}_s(\bm{s})$ to still satisfies the CBF conditions in critical points in which the effect of the virtual input $\bm{v}$ on the constraint function vanishes.

\subsection{Discrete-time Implementation}
\label{sec:DiscreteTime}
So far, we have assumed that the control synthesis problem in trajectory spaces can be solved in continuous time. However, although QPs can be solved efficiently, there will always be a limited control rate of the presented method. Therefore, we draw motivation from existing works on CBFs in sampled-data systems \cite{breeden2021control}. 

% \subsubsection{Lipschitz-based approach}

Assuming a zero-order-hold (ZOH) controller, i.e. $\forall t \in [t_k, t_k + \delta t)$
\begin{align*}
    \dot{\mathcal{T}}_u(t) = \left\{\bm{v}_i(t_k)\mid \forall \tau \in [\tau_i, \tau_{i+1}), \hspace{0.1cm} \forall i = 0, \dots, N-1\right\},
\end{align*}
we can modify the obtained CBF conditions to account for the worst-case evolution in between time steps $\delta t$. To that end, we can modify the right-hand-side (RHS) of Eq. \eqref{eq:SafeDT} and \eqref{eq:ActuationDT} to include an additional term of the form
\begin{align*}
    RHS := - \alpha \left(h(\bm{s})\right) + \nu\left(\delta t, \bm{s}\right)
\end{align*}

where $\nu$ is typically a function of Lipschitz constants of the dynamical system as well as the sampling time $\delta t$. One instantiation of $\nu$ is presented in \cite{breeden2021control} based on global Lipschitz constants of dynamics and constraint functions.

Although this approach enforces the forward invariance, the Lipschitz constants for high-dimensional trajectory space dynamics are generally hard to compute and lead to overly conservative behavior.

In practice, we aim to ensure that $h(\bm{s}(t_k+\delta t)) \geq 0$ for any $t_k\geq t_0$ given $h(\bm{s}(t_0)) \geq 0$. Approximating $h(\bm{s}(t_k+\delta t))$ by a first order estimate and using $h(t) = h(\bm{s}(t))$, we obtain
\begin{align*}
    h(t_k+\delta t) \approx h(t_k) + \dot{h}(t_k) \delta t  \geq 0 \\ \Leftrightarrow 
    \dot{h}(t_k)  \geq - \frac{1}{\delta t}h(t_k),
\end{align*}
where the estimate becomes exact as $\delta t \to 0$. When the approximation is \textit{good enough} with a given $\delta t>0$, we obtain an upper bound for the class-$\kappa$ functions $\alpha(x)$ in the inequality constraints \eqref{eq:SafeDT} and \eqref{eq:ActuationDT} as
\begin{align*}
    \dot{h}(t_k) \geq - \alpha(h(t_k)) \geq - \frac{1}{\delta t}h(t_k).
\end{align*}

Therefore, with real-time control ($\delta t \leq \SI{0.01}{\second}$) of mechanical systems that tend to be much slower, choosing
% \begin{align}
%     \alpha(x) \leq \frac{0.1}{\delta t}(x)
% \end{align}
% \begin{align}
    $\alpha(h) \ll \frac{1}{\delta t}h$
% \end{align}
is likely a safe choice.

\subsection{Limitations}
\label{sec:Limitations}
Our proposed method has two main limitations: 1) The provided guarantees hold in the limit as the replanning step approaches zero and 2) our method can be considered to be a shooting method, thus, FITS can have high sensitivity in the gradients for long planning horizons. 

As discussed in section~\ref{sec:DiscreteTime}, we can obtain Lipschitz arguments to account for the worst-case evolution assuming a ZOH discretization to address the first limitation.
While this seems promising, it is also difficult to find tight Lipschitz arguments for these high-dimensional systems. 

For the latter limitation, we can draw motivation from multiple shooting optimization methods which address the problem of gradient sensitivity. 
It might be possible to split the overall optimization problem into multiple smaller ones and enforce continuity between the optimized trajectories.
We are planning to address both limitations in future work.

\section{Connections to I-CBFs and NMPC}
In this section, we discuss how FITS is related to different control strategies, namely Integral CBFs (I-CBFs) and Sequential Quadratic Programming (NMPC-SQP). We present different angles on the control synthesis problem and position FITS in relation to other methods.
\input{content/Connections_B}

\section{Experiments}
In this section, we evaluate FITS in safety-critical scenarios in which a planning horizon is crucial. Specifically, we show
\begin{enumerate}
    \item adherence to safety specifications while not sacrificing performance in Sec.~\ref{sec:Quadrotor}
    \item that FITS overcomes the purely reactive behavior of CBFs in Sec.~\ref{sec:Clutter}.
    \vspace{-0.3cm}
\end{enumerate}
\begin{figure}[t]
    \centering
    \includegraphics[scale=0.55]{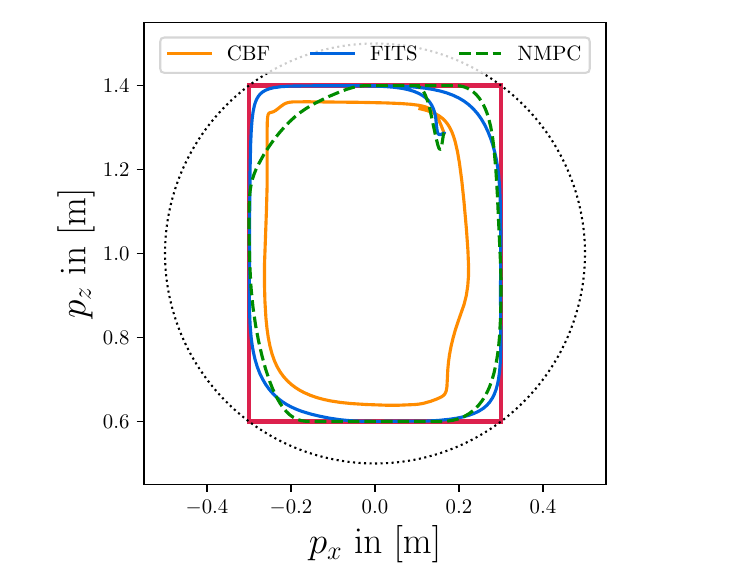}
    \caption{A two-dimensional quadrotor tracking task. The dotted black circle depicts the circular trajectory that we want to track while the red rectangle shows the constrained area that we want to remain in.}
    \label{fig:Tracking}
    \vspace{-0.6cm}
\end{figure}
\subsection{Setup}
To analyze FITS, we run experiments in the common benchmark simulation environment \texttt{safe-control-gym} \cite{yuan2022safe} which offers multiple scenarios for control of nonlinear systems under safety constraints\footnote{Code available at \href{https://github.com/mattivahs/FITS}{https://github.com/mattivahs/FITS}}. Specifically, we consider a quadrotor tracking task as well as a custom designed navigation task in a cluttered environment.

We consider two baselines, namely (CBF) a vanilla second order CBF in combination with an LQR reference controller and (NMPC) a nonlinear model predictive controller that recursively solves an optimization problem to obtain a discrete-time control sequence. We use the available implementation in \cite{yuan2022safe} which is based on CasADi's \texttt{opti} framework to formulate the optimization problem which is then solved with the interior point method IPOPT.

For our method, we use JAX \cite{jax2018github} for efficient auto differentiation to implement the differentiable integrator $\varphi$ in Eq.~\eqref{eq:ODEInt} as an Euler integration scheme. 
% Specifically, we use a simple Euler integration scheme but more advanced Runge-Kutta methods could be deployed as well.
Further, we use \texttt{cvxopt} to solve the QP in Eq.~\eqref{eq:QPTrajectorySpace}. All calculations (including auto differentiation) are carried out on an Intel i9-10900K processor and 32 GB of RAM.
\vspace{-0.1cm}
\subsection{Quadrotor Geofencing}
\label{sec:Quadrotor}
In the first experiment, we consider a circular trajectory tracking task of a planar quadrotor subject to a geofencing constraint of staying within a rectangular area as shown in Fig.~\ref{fig:Tracking}. The state of the robot is given as $\xb = [p_x~p_z~\theta~\dot{p}_x~\dot{p}_z~\dot{\theta}]^T$ and the control inputs are two thruster forces $\ub=[F_1~F_2]^T$. A detailed description of the equations of motion can be found in \cite{yuan2022safe}. The objective is of linear quadratic form
\begin{align}
    J\left(\mathcal{T}_x^I, \mathcal{T}_u^I\right) = \int_0^T \Delta \xb_{\tilde{\tau}}^T \bm{Q}_J \Delta \xb_{\tilde{\tau}} + \Delta \ub_{\tilde{\tau}}^T \bm{R}_J \Delta \ub_{\tilde{\tau}} \text{d}\tilde{\tau}\label{eq:QuadraticCost}
\end{align}
where $\Delta \xb_{\tilde{\tau}}= \xb_{\tilde{\tau}} - \xb_{\textrm{ref}}$ and $\Delta \ub_{\tilde{\tau}} = \ub_{\tilde{\tau}} - \ub_{\textrm{ref}}$, respectively. The safety objective is encoded as
\begin{align}
    h_i(\xb) = \bm{a}_i^T \xb - b_i \geq 0, \hspace{0.2cm} i = 1,\dots, 4 \label{eq:SafeSetSim1}
\end{align}
which represent the four halfspace constraints and actuation constraints are set to $F_i \in [0.056, 0.297]~\text{N}$. For the CBF baseline, we use the CBF design proposed in \cite{wu2016safety} which specifically deals with obstacle avoidance for planar quadrotors by explicitly including an angular dependency.

In both, NMPC and FITS, we optimize over $N = 20$ control steps where each control input is applied for $\delta \tau = 1/100 \text{s}$. 
The continuous state trajectory in FITS is discretized into $M=40$ states and the class-$\kappa$ functions in Eq.~\eqref{eq:QPTrajectorySpace} are chosen as $\alpha_1(x)=\alpha_2(x)=5x$ which satisfy the practical conditions obtained in Sec.~\ref{sec:DiscreteTime}. In experiments, we noticed that NMPC sometimes becomes infeasible which is why we had to introduce slack variables on the constraints in Eq.~\eqref{eq:SafeSetSim1}. 
{\renewcommand{\arraystretch}{1.1}
\begin{table}[t]
    \caption{Comparison of FITS to baselines.}
    \resizebox{0.47\textwidth}{!}{\begin{tabular}{c | ccccc}
                   & \begin{tabular}[c]{@{}c@{}}RMSE in [m] \\ $\mu \pm \sigma$\end{tabular} & \begin{tabular}[c]{@{}c@{}}$t_{\text{comp}}$  in [ms] \\ $(\mu \pm \sigma)$\end{tabular} & \begin{tabular}[c]{@{}c@{}} $\lVert\bm{u}\rVert_2$ in [N]\\ $(\mu \pm \sigma)$\end{tabular} & \#violations &
                   \begin{tabular}[c]{@{}c@{}} $h_{\textrm{min}}$\\ in [mm]\end{tabular}\\
                   \hline\hline
    FITS  &     \begin{tabular}[c]{@{}c@{}}$0.157 \pm$ \\ $0.075$\end{tabular}      &   \begin{tabular}[c]{@{}c@{}}$3.2 \pm$ \\ 2.5\end{tabular}    &      \begin{tabular}[c]{@{}c@{}}$0.189 \pm$ \\ $0.032$\end{tabular}     &      0  &0.2 \\
    \hline
    CBF &     \begin{tabular}[c]{@{}c@{}}$0.27 \pm$ \\ $0.04$\end{tabular}       &    \begin{tabular}[c]{@{}c@{}}$1.22 \pm$ \\ 0.3\end{tabular}  &      \begin{tabular}[c]{@{}c@{}}$0.21 \pm$ \\ $0.07$\end{tabular}    & 0& 43\\
    \hline
    NMPC   &      \begin{tabular}[c]{@{}c@{}}$0.153 \pm$ \\ $0.074$\end{tabular}    &   \begin{tabular}[c]{@{}c@{}}$98.3 \pm$ \\ 7.8\end{tabular}  &    \begin{tabular}[c]{@{}c@{}}$0.188 \pm$ \\ $0.06$\end{tabular}     &    27&-0.09
    \end{tabular}}
    \label{tab:simresults}
    \vspace{-0.5cm}
\end{table}
}

\vspace{-0.3cm}
Figure~\ref{fig:Tracking} shows the resulting trajectories for FITS and the baselines. FITS and CBF can ensure safety at all times which is further supported by Table~\ref{tab:simresults}. However, CBF is overly conservative and does not approach the boundaries of the safe set which results in poor tracking performance as indicated by a large root mean squared error (RMSE) on the tracked circular trajectory.
While it looks like NMPC satisfies safety constraints, violation occurs numerous times due to the slack variables in the optimization problem  with a maximum violation of $0.09$\si{mm}. 
FITS satisfies safety~and input constraints at all times while achieving similar performance to NMPC in terms of tracking, as indicated by the RMSE, and control effort evaluated as the mean applied force.

Apart from adhering to safety specifications, we observed the following two points that are worth highlighting:
\begin{enumerate}
    \item In the design process of safety functions for FITS, we neither explicitly accounted for higher relative degrees nor for the underactuation of the system since we directly used Eq.~\eqref{eq:SafeSetSim1} as candidate functions. The baseline CBF proposed in \cite{wu2016safety} is a second order CBF modified for the specific use case of planar quadrotors.
    \item The computation time of FITS was an order of magnitude lower than NMPC and could be run at 200 Hz since it only involved solving a single QP.
\end{enumerate}
\subsection{Navigation in Cluttered Environments}
\label{sec:Clutter}
In this section, we consider a navigation task of a double integrator system that is described by the dynamical system $[\dot{p}_x~\dot{p}_y~\ddot{p}_x~\ddot{p}_y]^T= [\dot{p}_x~\dot{p}_y~u_x~u_y]^T$. In this scenario, the robot has a goal reaching task while avoiding 30 circular obstacles as depicted in Fig.~\ref{fig:Clutter}. The cost is defined as in Eq.~\eqref{eq:QuadraticCost} and safety is encoded as
\begin{align}
    h_i\ofx = \lVert [p_x~p_y]^T - \bm{p}_o^{(i)} \rVert_2 - r_o^{(i)}, \hspace{0.3cm} \forall i = 1,\dots, 30\label{eq:SafetySpecSim2}
\end{align}
where each obstacle is defined by its position $\bm{p}_o$ and its radius $r_o$. The actuation limits are $u_x, u_y \in [-1, 1]$ and the simulation time is $t_{\textrm{sim}}=8$ seconds. 

In this simulation, we consider two CBFs with different class-$\kappa$ functions $\alpha_1(x) = 20x$ and $\alpha_2(x)=2x$. It can be observed that FITS and NMPC are able to reach the goal in the limited time frame while both CBFs get stuck. Interestingly, the more conservative CBF $\alpha_2$ avoids getting stuck directly but moves slowly through the environment because the low value of $\alpha_2$ results in more conservative behavior. NMPC, on the other hand, manages to find a path through the obstacles but violates safety specifications multiple times due to slack variables in the optimization problem. Specifically, the minimum value of the safety functions in Eq.~\eqref{eq:SafetySpecSim2} is $h=- 0.026$\si{m} while both, CBF and FITS, keep the state safe at all times.
\section{Conclusion and Future Work}
% In this work, we have introduced 
In this work, we introduce forward invariance in trajectory spaces (FITS) as an approach to safety-critical control. 
We model planned trajectories as a dynamical system in a trajectory configuration space $\mathcal{S}$ and define safe sets in this space that encode safety of physical states and actuation constraints. We find an efficient QP formulation for continuous-time receding horizon control by defining CBFs in the space $\mathcal{S}$ that guarantees safe evolution of planned trajectories. 
Consequently, by incorporating a planning horizon, we overcome purely reactive behavior of standard CBF-QP formulations. Experimental results indicate that we achieve adherence to safety specifications while offering a considerable speedup in computation time compared to an interior point NMPC. 
Future work will focus on the limitations pointed out in Sec.~\ref{sec:Limitations}, namely the problem of missing rigorous guarantees for discrete-time implementations and on the gradient sensitivity of FITS for long planning horizons.
Lastly, with FITS, we show that CBFs are not limited to only physical state spaces and can be extended to alternative domains. We believe this perspective can be highly useful in applications beyond safety-critical control. 
\begin{figure}[t]
    \centering
    \includegraphics[scale=0.67]{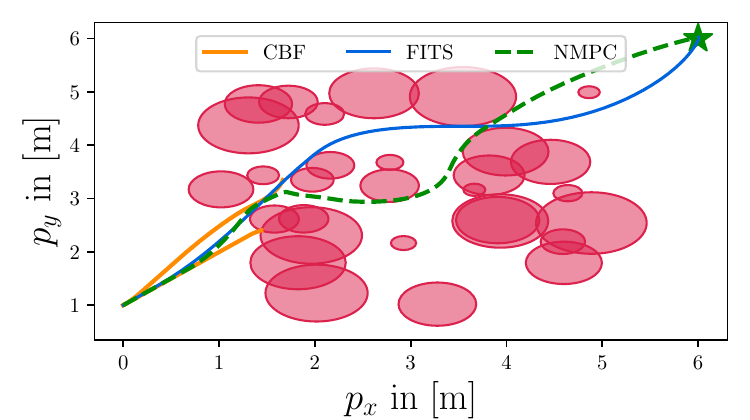}
    \caption{A two-dimensional navigation task. Obstacles are shown by red circles while the goal position is indicated by the green star.}
    \label{fig:Clutter}
    \vspace{-0.55cm}
\end{figure}
\balance
\hypersetup{urlcolor=black}
\bibliographystyle{IEEEtran}
\bibliography{references.bib}

\end{document}

%% file: content/preliminaries_C.tex
\section{Background}
In this section we develop the notation used throughout the paper and present the necessary technical background to formalize the problem.

Consider a general nonlinear control system defined by the ordinary differential equation (ODE) 
\begin{align}
    \ddtau \xb_{\tau}  = \bm{F}(\xb_{\tau}, \ub_{\tau}), \label{eq:control_system_dynamics}
\end{align}
with state $\xb \in \mathcal{X} \subseteq \mathbb{R}^{n_x}$, input $\ub \in \mathcal{U} \subseteq \mathbb{R}^{n_u}$ and a continuously differentiable deterministic function $\bm{F}: \mathcal{X} \times \mathcal{U} \to \mathbb{R}^{n_x}$. 
By encoding safety through a continuously differentiable constraint function $h: \mathcal{X} \rightarrow \mathbb{R}$, a safe set $\mathcal{C} \subseteq \mathcal{X}$ can be constructed as the 0-superlevel set of $h$ such that
\begin{equation}
\begin{aligned}
  \label{eq:safeset_state}
    \mathcal{C} &= \left\{\bm{x} \in \mathcal{X} \mid h\left(\bm{x}\right) \geq 0\right\},\\
    \partial \mathcal{C} &= \left\{\bm{x} \in \mathcal{X} \mid h\left(\bm{x}\right) = 0\right\}.
 \end{aligned}
\end{equation} 
\begin{definition}
A safe set $\mathcal{C}$ is forward invariant with respect to the system \eqref{eq:control_system_dynamics} if for every initial condition $\bm{x}_{\tau_0} \in \mathcal{C}$
it holds that $\bm{x}_\tau \in \mathcal{C}, \forall \tau \geq \tau_0$.
\end{definition}
A prominent approach in safe control is to render safe sets \textit{forward invariant} by using CBFs.
\begin{definition}
\label{def:CBF}
Given a set $\mathcal{C}$, defined by Eq. \eqref{eq:safeset_state}, $h$ serves as a zeroing CBF for the system \eqref{eq:control_system_dynamics} if  for all $\bm{x}$ satisfying $h\left(\bm{x}\right) \geq 0$, there exists a $\bm{u} \in \mathcal{U}$ such that
\begin{align*}
    \ddtau h(\xb_{\tau}) = \frac{\partial h}{\partial \bm{x}} \left(\bm{F}\left(\bm{x}_{\tau}, \bm{u}_{\tau}\right)\right) \geq - \alpha\left(h\left(\bm{x}_{\tau}\right)\right),
\end{align*}
where $\alpha$ is a class-$\kappa$ function,
% \footnote{Continuous and strictly increasing with $\alpha(0)=0$.}
i.e. continuous and strictly increasing with $\alpha(0)=0$.
% .
\end{definition}
 If a valid CBF exists, it follows that a controller satisfying Def.~\ref{def:CBF} renders $\mathcal{C}$ forward invariant \cite{ames2016control}. In the case of control affine dynamics of the form $\bm{F}(\xb, \ub) = \bm{f}\ofx + \bm{g}\ofx \ub$, we can formulate the control synthesis as quadratic program (QP)
\begin{equation}
\begin{aligned}
    \bm{u}_{\tau}^* = \argmin_{\ub_{\tau} \in \mathcal{U}} \quad & \left(\ub_{\tau} - \ub_{\text{ref}}\right)^T {\bm{Q}} \left(\ub_{\tau} - \ub_{\text{ref}}\right)\\
    \textrm{s.t.~~} \quad & \frac{\partial h}{\partial \xb} \left(\bm{f}(\xb_{\tau}) + \bm{g}(\xb_{\tau}) \ub_{\tau} \right) \geq - \alpha(h(\xb_{\tau})),
\end{aligned}
\label{eq:QP}
\end{equation}
where $\ub_{\text{ref}}$ is a reference controller and $\bm{Q}$ denotes a weighting matrix.
\begin{theorem}
    \label{thm:CBF}
    If the QP in Eq. \eqref{eq:QP} is always feasible, the set $\mathcal{C} = \left\{\xb \in \mathcal{X} \mid h\ofx \geq 0\right\}$ is forward invariant, i.e. $\xb_{\tau_0} \in \mathcal{C} \implies \xb_\tau \in \mathcal{C} ~~\forall \tau \geq \tau_0$.
\end{theorem}
\noindent For a detailed proof, the reader is referred to \cite{ames2016control}.

\subsection{Input and State Trajectories}
We define the state trajectory $\mathcal{T}_x^{I}$ and the input trajectory $\mathcal{T}_u^{I}$ on a closed interval $I=[a, b]$ as
\begin{align*}
    \mathcal{T}_x^I &:= \{\bm{x}_\tau \in \mathcal{X} \mid \tau \in I \} \in \mathcal{X}^I \subseteq L^2(I, \mathcal{X}) \\
    \mathcal{T}_u^I &:= \{\ub_\tau \in \mathcal{U} \mid \tau \in I \} \in \mathcal{U}^I \subseteq L^2(I, \mathcal{U}),
\end{align*}
where $\mathcal{X}^I$ and $\mathcal{U}^I$ are the sets of possible trajectories that are also \textit{square integrable} ($L^2$) on the interval $I$. 

A state $\xb_\tau$ at time $\tau \in I$ is fully determined by the dynamics \eqref{eq:control_system_dynamics}, the initial condition $\xb_a \in \mathcal{X}$ and the input trajectory $\mathcal{T}_u^{[a, \tau]}$ as
\begin{align}
    \xb_\tau = \varphi\left(\xb_a, \mathcal{T}_u^{[a, \tau]}\right) \coloneq \xb_a + \int_{a}^\tau \bm{F}\left(\xb_{\tilde{\tau}}, \ub_{\tilde{\tau}}\right) \mathrm{d}\tilde{\tau},     \label{eq:ODEInt}
\end{align}
thus the state trajectory can be obtained by integrating the ODE in Eq.~\eqref{eq:control_system_dynamics}, i.e.
\begin{align*}
    \mathcal{T}_x^{I} &= \left\{\varphi\left(\xb_a, \mathcal{T}_u^{[a, \tau]}\right) \in \mathcal{X} \mathrel{\Big|} \tau \in I \right\}.
\end{align*}  
Note that state trajectories are continuous, i.e. $\mathcal{X}^I \subseteq C^0(I, \mathcal{X})$ for any square integrable input trajectory, which follows from the continuous differentiability assumption on the dynamics \eqref{eq:control_system_dynamics}.

% Since trajectories can vary in time, 
We introduce the variable $t$ to denote \textit{system time} which describes the temporal evolution in the real world. 
% Finally, we introduce and highlight the difference between two time variables: \textit{trajectory time} $\tau$ and \textit{system time} $t$.  
In contrast, we use \textit{trajectory time} $\tau$, to describe the time evolution of a state $\xb_{\tau}$ along a state trajectory $\mathcal{T}_x^I$ within the interval~$I$.

% A state trajectory $\mathcal{T}_x^I$ describes the evolution of a state $\xb$ along \textit{trajectory time} $\tau$ within a planning horizon $I=[0, T]$. 
% A state trajectory can evolve along the \textit{system time} $t$ .

Consequently, the state and input trajectories become a function of system time, i.e. $\mathcal{T}_x^I(t)$ and $\mathcal{T}_u^I(t)$, respectively. 
Figure~\ref{fig:state_trajectory_space} illustrates a one-dimensional example of a state trajectory $\mathcal{T}_x^I(t)$ varying over system time $t$.
In the remainder of this paper, we use subscripts, i.e. $\xb_{\tau}$, to denote trajectory time and parentheses, i.e. $\xb(t)$, to denote system time. 
As also visualized in Figure~\ref{fig:state_trajectory_space}, the state $x_{\tau_p}(t_p)$ denotes the state at system time $t_p$ and trajectory time $\tau_p$.

%% file: content/Connections_B.tex
\subsection{Integral Control Barrier Functions}
I-CBFs\cite{ames2020integral} were introduced for the coupled dynamical system
\begin{align*}
    \begin{bmatrix}
        \dot{\xb}\\
        \dot{\ub}
    \end{bmatrix} = \begin{bmatrix}
        \bm{F}(\xb, \ub)\\
        \phi(\xb, \ub, t) + \bm{v}
    \end{bmatrix}
\end{align*}
where $\ub, \bm{v} \in \mathbb{R}^{n_u}$ and $\phi$ is a general feedback control law. Similar to FITS, the authors rely on a control law through a virtual input $\bm{v}$. The main difference to FITS lies in the fact that we consider an input trajectory $\mathcal{T}^I_u$ instead of just one control input. Therefore, our method overcomes the short-sightedness of I-CBFs by incorporating a planning horizon.

The authors in \cite{ames2020integral} formulate a QP with $\bm{v}$ as decision variable to ensure forward invariance of a safe set $\mathcal{C}$, also incorporating input constraints. In FITS, we can view the unconstrained minimization of $\dot{J}(\bm{s})$ in Eq.~\eqref{eq:CostDT} as the general feedback law $\phi$ which describes a gradient descent policy on the objective function $J$. We recover the exact I-CBF formulation by reducing the planning horizon to a single control input $\ub_0$, resulting in a myopic control law.

\subsection{Sequential Quadratic Programming for NMPC}
The discrete implementation of FITS is tightly connected to real-time iteration NMPC-SQP methods \cite{zanelli2021inexact, numerow2024inherently, gros2020linear}. Both are solving a quadratic program at each iteration. FITS solves for rates; SQP solves for deltas.
% Computing discrete FITS steps (constrained gradient descent) until convergence is essentially solving a shooting based SQP until convergence. Both are local methods and using only input trajectories as optimization variables, so in convex cost regions they should also converge to the minimum set. Realtime NMPC-SQP methods, solve one (or few) SQP iterations per control time step. This is analogous to integrating one (or few) FITS steps per control step.
One way how discrete FITS steps generalize SQP iterates, is that the latter use linearized constraints which resemble the CBF inequality constraints in Eq.~\eqref{eq:SafeDT} with a fixed class-$\kappa$ function $\alpha(h)=1 \cdot h$.

While SQP iterates are not ensured to remain safe, its iterates are equivalent to the discrete FITS steps when using a constant symmetric positive definite (SPD) hessian approximation and the above choice of class-$\kappa$ function. 
% This has implications in both directions: FITS iterates might violate constraints when Lipschitz bounds are not considered, and SQP iterates can be ensured to remain valid (satisfy constraints) when considering Lipschitz bounds. 
Following the analysis in Sec.~\ref{sec:DiscreteTime}, our work provides insight into when real-time iteration NMPC-SQP methods ensure safety in practice, since the iteration time is low and the linearized constraint function resembles a CBF constraint.

%% file: main.bbl
\begin{thebibliography}{10}
\providecommand{\url}[1]{#1}
\csname url@samestyle\endcsname
\providecommand{\newblock}{\relax}
\providecommand{\bibinfo}[2]{#2}
\providecommand{\BIBentrySTDinterwordspacing}{\spaceskip=0pt\relax}
\providecommand{\BIBentryALTinterwordstretchfactor}{4}
\providecommand{\BIBentryALTinterwordspacing}{\spaceskip=\fontdimen2\font plus
\BIBentryALTinterwordstretchfactor\fontdimen3\font minus \fontdimen4\font\relax}
\providecommand{\BIBforeignlanguage}[2]{{%
\expandafter\ifx\csname l@#1\endcsname\relax
\typeout{** WARNING: IEEEtran.bst: No hyphenation pattern has been}%
\typeout{** loaded for the language `#1'. Using the pattern for}%
\typeout{** the default language instead.}%
\else
\language=\csname l@#1\endcsname
\fi
#2}}
\providecommand{\BIBdecl}{\relax}
\BIBdecl

\bibitem{ames2016control}
A.~D. Ames, X.~Xu, J.~W. Grizzle, and P.~Tabuada, ``Control barrier function based quadratic programs for safety critical systems,'' \emph{Transactions on Automatic Control}, vol.~62, no.~8, pp. 3861--3876, 2016.

\bibitem{garg2024advances}
K.~Garg, J.~Usevitch, J.~Breeden, M.~Black, D.~Agrawal, H.~Parwana, and D.~Panagou, ``Advances in the theory of control barrier functions: Addressing practical challenges in safe control synthesis for autonomous and robotic systems,'' \emph{Annual Reviews in Control}, vol.~57, p. 100945, 2024.

\bibitem{grune2017nonlinear}
L.~Gr{\"u}ne, J.~Pannek, L.~Gr{\"u}ne, and J.~Pannek, \emph{Nonlinear model predictive control}.\hskip 1em plus 0.5em minus 0.4em\relax Springer, 2017.

\bibitem{breeden2022predictive}
J.~Breeden and D.~Panagou, ``Predictive control barrier functions for online safety critical control,'' in \emph{Conference on Decision and Control (CDC)}.\hskip 1em plus 0.5em minus 0.4em\relax IEEE, 2022, pp. 924--931.

\bibitem{infinitesimal2024}
P.~Jang, Inkyu, S.~Hwang, J.~Byun, and J.~H. Kim, ``Safe receding horizon motion planning with infinitesimal update interval,'' in \emph{International Conference on Robotics and Automation}.\hskip 1em plus 0.5em minus 0.4em\relax IEEE, 2024.

\bibitem{wabersich2022predictive}
K.~P. Wabersich and M.~N. Zeilinger, ``Predictive control barrier functions: Enhanced safety mechanisms for learning-based control,'' \emph{Transactions on Automatic Control}, vol.~68, no.~5, 2022.

\bibitem{zeng2021safety}
J.~Zeng, B.~Zhang, and K.~Sreenath, ``Safety-critical model predictive control with discrete-time control barrier function,'' in \emph{American Control Conference (ACC)}.\hskip 1em plus 0.5em minus 0.4em\relax IEEE, 2021, pp. 3882--3889.

\bibitem{vahs2023belief}
M.~Vahs, C.~Pek, and J.~Tumova, ``Belief control barrier functions for risk-aware control,'' \emph{IEEE Robotics and Automation Letters}, 2023.

\bibitem{vahs2024risk}
M.~Vahs and J.~Tumova, ``Risk-aware control for robots with non-gaussian belief spaces,'' in \emph{International Conference on Robotics and Automation (ICRA)}.\hskip 1em plus 0.5em minus 0.4em\relax IEEE, 2024.

\bibitem{black2023future}
M.~Black, M.~Jankovic, A.~Sharma, and D.~Panagou, ``Future-focused control barrier functions for autonomous vehicle control,'' in \emph{American Control Conference (ACC)}.\hskip 1em plus 0.5em minus 0.4em\relax IEEE, 2023, pp. 3324--3331.

\bibitem{gurriet2018online}
T.~Gurriet, M.~Mote, A.~D. Ames, and E.~Feron, ``An online approach to active set invariance,'' in \emph{Conference on Decision and Control (CDC)}.\hskip 1em plus 0.5em minus 0.4em\relax IEEE, 2018, pp. 3592--3599.

\bibitem{chen2021backup}
Y.~Chen, M.~Jankovic, M.~Santillo, and A.~D. Ames, ``Backup control barrier functions: Formulation and comparative study,'' in \emph{Conference on Decision and Control (CDC)}.\hskip 1em plus 0.5em minus 0.4em\relax IEEE, 2021, pp. 6835--6841.

\bibitem{wiltz2023construction}
A.~Wiltz, X.~Tan, and D.~V. Dimarogonas, ``Construction of control barrier functions using predictions with finite horizon,'' in \emph{Conference on Decision and Control (CDC)}.\hskip 1em plus 0.5em minus 0.4em\relax IEEE, 2023, pp. 2743--2749.

\bibitem{zeng2021enhancing}
J.~Zeng, Z.~Li, and K.~Sreenath, ``Enhancing feasibility and safety of nonlinear model predictive control with discrete-time control barrier functions,'' in \emph{Conference on Decision and Control (CDC)}.\hskip 1em plus 0.5em minus 0.4em\relax IEEE, 2021, pp. 6137--6144.

\bibitem{gros2020linear}
S.~Gros, M.~Zanon, R.~Quirynen, A.~Bemporad, and M.~Diehl, ``From linear to nonlinear mpc: bridging the gap via the real-time iteration,'' \emph{International Journal of Control}, vol.~93, no.~1, pp. 62--80, 2020.

\bibitem{zanelli2021inexact}
A.~Zanelli, ``Inexact methods for nonlinear model predictive control: stability, applications, and software,'' Ph.D. dissertation, Dissertation, Universit{\"a}t Freiburg, 2021, 2021.

\bibitem{numerow2024inherently}
L.~Numerow, A.~Zanelli, A.~Carron, and M.~N. Zeilinger, ``Inherently robust suboptimal mpc for autonomous racing with anytime feasible sqp,'' \emph{arXiv preprint arXiv:2401.02194}, 2024.

\bibitem{ohtsuka2004continuation}
T.~Ohtsuka, ``A continuation/gmres method for fast computation of nonlinear receding horizon control,'' \emph{Automatica}, vol.~40, no.~4, pp. 563--574, 2004.

\bibitem{shen2016modified}
C.~Shen, B.~Buckham, and Y.~Shi, ``Modified c/gmres algorithm for fast nonlinear model predictive tracking control of auvs,'' \emph{IEEE transactions on control systems technology}, vol.~25, no.~5, pp. 1896--1904, 2016.

\bibitem{kidger2021hey}
P.~Kidger, R.~T.~Q. Chen, and T.~J. Lyons, ``"hey, that's not an ode": Faster ode adjoints via seminorms.'' \emph{International Conference on Machine Learning}, 2021.

\bibitem{xu2018constrained}
X.~Xu, ``Constrained control of input--output linearizable systems using control sharing barrier functions,'' \emph{Automatica}, vol.~87, 2018.

\bibitem{boyd2004convex}
S.~Boyd and L.~Vandenberghe, \emph{Convex optimization}.\hskip 1em plus 0.5em minus 0.4em\relax Cambridge university press, 2004.

\bibitem{breeden2021control}
J.~Breeden, K.~Garg, and D.~Panagou, ``Control barrier functions in sampled-data systems,'' \emph{IEEE Control Systems Letters}, vol.~6, pp. 367--372, 2021.

\bibitem{ames2020integral}
A.~D. Ames, G.~Notomista, Y.~Wardi, and M.~Egerstedt, ``Integral control barrier functions for dynamically defined control laws,'' \emph{IEEE control systems letters}, vol.~5, no.~3, pp. 887--892, 2020.

\bibitem{yuan2022safe}
Z.~Yuan, A.~W. Hall, S.~Zhou, L.~Brunke, M.~Greeff, J.~Panerati, and A.~P. Schoellig, ``Safe-control-gym: A unified benchmark suite for safe learning-based control and reinforcement learning in robotics,'' \emph{IEEE Robotics and Automation Letters}, vol.~7, no.~4, pp. 11\,142--11\,149, 2022.

\bibitem{jax2018github}
\BIBentryALTinterwordspacing
J.~Bradbury, R.~Frostig, P.~Hawkins, M.~J. Johnson, C.~Leary, D.~Maclaurin, G.~Necula, A.~Paszke, J.~Vander{P}las, S.~Wanderman-{M}ilne, and Q.~Zhang, ``{JAX}: composable transformations of {P}ython+{N}um{P}y programs,'' \emph{\vphantom}, 2018. [Online]. Available: \url{http://github.com/google/jax}
\BIBentrySTDinterwordspacing

\bibitem{wu2016safety}
G.~Wu and K.~Sreenath, ``Safety-critical control of a planar quadrotor,'' in \emph{American control conference (ACC)}.\hskip 1em plus 0.5em minus 0.4em\relax IEEE, 2016, pp. 2252--2258.

\end{thebibliography}
